\documentclass[times, 10pt,twocolumn]{article}
\usepackage{latex12}
\usepackage{times}
\usepackage{balance}
\usepackage{graphicx, enumerate}
\usepackage{amsmath,amssymb, bm}
\usepackage{amsthm}
\newtheorem{theorem}{Proposition}

\usepackage{my-style}

\allowdisplaybreaks[4]

\begin{document}

\title{Composite Likelihood Estimation for Restricted Boltzmann machines}

\author{Muneki Yasuda$^*$, Shun Kataoka, Yuji Waizumi and Kazuyuki Tanaka\\
\emph{Graduate School of Science and Engineering, Yamagata University, Japan$^*$}\\
\emph{Graduate School of Information Sciences, Tohoku University, Japan}\\
\emph{muneki@yz.yamagata-u.ac.jp$^*$}\\
}

\maketitle
\thispagestyle{empty}

\begin{abstract}
Learning the parameters of graphical models using the maximum likelihood estimation is generally hard which requires an approximation. 
Maximum composite likelihood estimations are statistical approximations of the maximum likelihood estimation which are higher-order generalizations of the 
maximum pseudo-likelihood estimation. 
In this paper, we propose a composite likelihood method and investigate its property. 
Furthermore, we apply our composite likelihood method to restricted Boltzmann machines.
\end{abstract}

\Section{Introduction}

Learning the parameters of graphical models using maximum likelihood (ML) estimation is generally hard due to the intractability of computing the normalizing constant and its gradients.
Maximum pseudo-likelihood (PL) estimation~\cite{PL1975} is a statistical approximation of the ML estimation. Unlike the ML estimation, 
the maximum PL estimation is computationally fast, but however, the estimates obtained by this method are not very accurate.

\textit{Composite likelihoods} (CLs)~\cite{CLM1988} are higher-order generalizations of the PL. 
Asymptotic analysis shows that maximum CL estimation is statistically more efficient than the maximum PL estimation~\cite{Liang2008}. 
It has been known that the maximum PL estimation is asymptotically consistent~\cite{PL1975}. Like this, the maximum CL estimation is also asymptotically consistent~\cite{CLM1988}. 
Furthermore, the maximum CL estimation has an asymptotic variance that is smaller than the maximum PL estimation but larger than ML estimation~\cite{Liang2008,Dillon2010}. 
Recently, it has been found that the maximum CL estimation corresponds to a block-wise contrastive divergence learning~\cite{BlockCD2010}.

In the maximum CL estimation, one can freely choose the size of ``blocks'' which contain several variables, 
and it is widely believed that by increasing the size of blocks, one can capture more dependency relations in the model and increase the accuracy of the estimates~\cite{BlockCD2010}. 
In the first part of this paper, we introduce a systematic choice of blocks in the maximum CL estimation. 
In our proposed choice of blocks, it is guaranteed that one can obtain quantitatively closer value to the true likelihood by increasing the size of blocks.
In the latter part of this paper, we apply our maximum CL estimation to restricted Boltzmann machines (RBMs)~\cite{Hinton2002} 
and show results of numerical experiments using synthetic data.

\Section{Composite Likelihood Estimation}

For the $n$ dimensional discrete random variable $\bm{x} := \{x_i \mid i \in \Omega =\{ 1,2,\ldots, n\}\}$, let us consider the probabilistic model expressed as 
\begin{align}
P(\bm{x} \mid \bm{\theta} ):= Z(\bm{\theta})^{-1}\exp \big( - E(\bm{x} \mid \bm{\theta})\big),
\label{eq:P(x)}
\end{align}
where $E(\bm{x} \mid \bm{\theta})$ is the energy function having an arbitrary functional form and 
$
Z(\bm{\theta}):= \sum_{\bm{x}}\exp \big( - E(\bm{x} \mid \bm{\theta})\big)
$
is the normalizing constant. 
Let us suppose that the data set composed of $M$ data, $\mcal{D}:=\{\bm{d}^{(\mu)}\mid \mu = 1,2,\ldots, M\}$ is obtained. 
Each data is statistically-independent of each other.
In the perspective of the ML estimation, we determine the optimal $\bm{\theta}$ by maximizing the log-likelihood function defined by
\begin{align}
\mcal{L}_{\mrm{ML}}(\bm{\theta}):=\sum_{\bm{x}}Q(\bm{x}) \ln P(\bm{x} \mid \bm{\theta} ),
\label{eq:L_ML}
\end{align}
where $Q(\bm{x})$ is the empirical distribution of the data set, i.e. the histogram of data set, expressed by
$
Q(\bm{x}):=M^{-1}\sum_{\mu = 1}^M \delta(\bm{x},\bm{d}^{(\mu)})
$,
where we define
\begin{align*}
\delta(\bm{x},\bm{d}^{(\mu)})&:=
\begin{cases}
1 & \bm{x}=\bm{d}^{(\mu)}\\
0 & \bm{x}\not=\bm{d}^{(\mu)}
\end{cases}
.
\end{align*}
However, maximizing $\mcal{L}_{\mrm{ML}}(\bm{\theta})$ with respect to $\bm{\theta}$ is computationally expensive. 
This generally requires the computational cost of $O(e^n)$ due to multiple summations. 

The maximum CL estimation is a statistical approximation technique of the ML estimation~\cite{CLM1988}. 
In the maximum CL estimation, one divides $\Omega$ into some different subsets termed \textit{blocks}, $c_1, c_2,\ldots c_r \subseteq \Omega$, with allowing overlaps among blocks. 
Note that the relation $c_1 \cup  c_2 \cup \ldots \cup c_r = \Omega$ must be kept. 
We denote the family of these blocks, $c_1, c_2,\ldots c_r$, by $\mcal{F}$.
For the family $\mcal{F}$, the CL is defined by
\begin{align}
\mcal{L}_{\mcal{F}}(\bm{\theta}) := \Lambda_{\mcal{F}}  \sum_{c \in \mcal{F}} \sum_{\bm{x}}Q(\bm{x})\ln P(\bm{x}_{c} \mid \bm{x}_{\bar{c}}, \bm{\theta}),
\label{eq:CL}
\end{align}
where, for a set $A \subseteq \Omega$, the expression $\bm{x}_A$ is defined as $\bm{x}_A := \{x_i \mid i \in A \}$ and $\bar{A}:= \Omega \setminus A$. 
The notation $\Lambda_{\mcal{F}}$ is defined by $\Lambda_{\mcal{F}}:=|\mcal{F}|^{-1}$, 
where the notation $|\cdots|$ denotes the size of the assigned set.
From the Bayesian theorem, the conditional probability in the CL is obtained by
$
P(\bm{x}_{c} \mid \bm{x}_{\bar{c}}, \bm{\theta})= P(\bm{x} \mid \bm{\theta})\big(\sum_{\bm{x}_c}P(\bm{x} \mid \bm{\theta})\big)^{-1}
$. 
In the CL estimation, one maximizes the CL instead of the true log-likelihood. 
If each block is composed of just one variable, i.e. $c_i = \{i\}$ and $r=n$, the CL is reduced to the PL~\cite{PL1975}. 
Hence, the CL can be regarded as a generalization of the PL.
On the other hand, if $r=1$ and $c_1$ is composed of all variables, the CL is obviously equivalent to the true log-likelihood $\mcal{L}_{\mrm{ML}}(\bm{\theta})$.
\begin{theorem} \label{prop:upper}
The CL generally is an upper bound on the true log-likelihood $\mcal{L}_{\mrm{ML}}(\bm{\theta})$.
\end{theorem}
\begin{proof}
The relation between original log-likelihood and the CL can be expressed as
$
\mcal{L}_{\mrm{ML}}(\bm{\theta})=\mcal{L}_{\mcal{F}}(\bm{\theta})+\mcal{R}_{\mcal{F}}(\bm{\theta})
$,
where the remainder term is defined as
\begin{align}
\mcal{R}_{\mcal{F}}(\bm{\theta}):=
\Lambda_{\mcal{F}} \sum_{c \in \mcal{F}} \sum_{\bm{x}}Q(\bm{x})\ln \sum_{\bm{x}_{c}}P(\bm{x} \mid \bm{\theta}).
\label{eq:R_F}
\end{align}
Since $P(\bm{x} \mid \bm{\theta})$ is a discrete distribution and $\Lambda_{\mcal{F}}$ is positive, the remainder term, $\mcal{R}_{\mcal{F}}(\bm{\theta})$, is less than or equal to zero.
Therefore, the inequality $\mcal{L}_{\mrm{ML}}(\bm{\theta})\leq \mcal{L}_{\mcal{F}}(\bm{\theta})$ is generally satisfied for any $\bm{\theta}$ and for any choice of $\mcal{F}$.
\end{proof}

\SubSection{Systematic Choice of Blocks}

In this section, we introduce a particular choice of the blocks in which the CL has a good property. 
For $1\leq k\leq n$, we define the family $\mcal{F}_k$ whose elements are all possible blocks composed of $k$ different variables, i.e.
$\mcal{F}_k:=\{\{i_1,i_2,\ldots,i_k\} \mid i_1 < i_2 < \cdots < i_k \in \Omega\}$. 
For example, when $n=4$, $\mcal{F}_2 = \{ \{ 1,2\},\{1,3\}, \{1,4\}, \{2,3\},\{2,4\},\{3,4\}\}$ and 
$\mcal{F}_3 = \{ \{ 1,2,3\},\{1,2,4\}, \{1,3,4\}, \{2,3,4\}\}$. 
For the family $\mcal{F}_k$, the CL is expressed as
\begin{align}
\mcal{L}_{\mcal{F}_k}(\bm{\theta})=\Lambda_{\mcal{F}_k}\sum_{c \in \mcal{F}_k} \sum_{\bm{x}}Q(\bm{x})\ln P(\bm{x}_{c} \mid \bm{x}_{\bar{c}}, \bm{\theta}),
\label{eq:CL-k}
\end{align}
where $\Lambda_{\mcal{F}_k}=|\mcal{F}_k|^{-1} = k!(n-k)!/n!$.
It is noteworthy that $\mcal{L}_{\mcal{F}_1}(\bm{\theta})$ is reduced to the PL and $\mcal{L}_{\mcal{F}_n}(\bm{\theta})=\mcal{L}_{\mrm{ML}}(\bm{\theta})$. 
\begin{theorem} \label{prop:mono}
For $1\leq k\leq n$, the CLs for the family $\mcal{F}_k$ is bounded as
$
\mcal{L}_{\mcal{F}_1}(\bm{\theta})\geq  \mcal{L}_{\mcal{F}_2}(\bm{\theta})\geq \cdots \geq \mcal{L}_{\mcal{F}_n}(\bm{\theta})=\mcal{L}_{\mrm{ML}}(\bm{\theta})
$ for any $\bm{\theta}$.
\end{theorem}
\begin{proof}
The relation between $\mcal{L}_{\mrm{ML}}(\bm{\theta})$ and $\mcal{L}_{\mcal{F}_k}(\bm{\theta})$ is 
$
\mcal{L}_{\mrm{ML}}(\bm{\theta}) = \mcal{L}_{\mcal{F}_k}(\bm{\theta}) + \mcal{R}_{\mcal{F}_k}(\bm{\theta})
$,
where, from equation (\ref{eq:R_F}) the remainder term is 
$
\mcal{R}_{\mcal{F}_k}(\bm{\theta})
=\Lambda_{\mcal{F}_k}\sum_{c \in \mcal{F}_k} \sum_{\bm{x}}Q(\bm{x})\ln \sum_{\bm{x}_c }P(\bm{x} \mid \bm{\theta})
$.
Let us consider the difference between the the remainder terms $D_k(\bm{\theta}):=\mcal{R}_{\mcal{F}_{k+1}}(\bm{\theta})
-\mcal{R}_{\mcal{F}_k}(\bm{\theta})$. 
After a short manipulation, the difference $D_k(\bm{\theta})$ yields
\begin{align*}
&D_k(\bm{\theta})\nn
&=\frac{\Lambda_{\mcal{F}_k}}{k-n}\sum_{c \in \mcal{F}_k} \sum_{i \in \bar{c}}
\sum_{\bm{x}}Q(\bm{x}) \ln\frac{\sum_{\bm{x}_c}P(\bm{x} \mid \bm{\theta})}{ \sum_{\bm{x}_c,x_i}P(\bm{x} \mid \bm{\theta})}.
\end{align*}
Hence, for $1\leq k\leq n-1$, the inequality $D_k(\bm{\theta})\geq 0$ holds, because $P(\bm{x} \mid \bm{\theta})$ is a discrete distribution. 
Therefore, for $1\leq k\leq n-1$, the inequality
\begin{align}
\mcal{L}_{\mcal{F}_k}(\bm{\theta})\geq \mcal{L}_{\mcal{F}_{k+1}}(\bm{\theta})
\label{eq:bound1}
\end{align}
is satisfied.
From equation (\ref{eq:bound1}), we reach to the proposition.
\end{proof}
From propositions \ref{prop:upper} and \ref{prop:mono}, we found that, for $1\leq k\leq n-1$, 
the $k$th-order CL, $\mcal{L}_{\mcal{F}_k}(\bm{\theta})$, is always an upper bound on the true log-likelihood and 
it monotonically approaches the true log-likelihood with the increase of $k$. 
Therefore, it is guaranteed that a larger $k$ gives quantitatively better approximation of the true log-likelihood.

\Section{Application to Restricted Boltzmann Machines} \label{sec:RBM}

In this section, we apply the CL estimation to RBMs~\cite{Hinton2002}. 
RBMs are Boltzmann machines consisting of visible variables, whose states can be observed, 
and hidden variables, whose states are not specified by observed data. 
RBMs are defined on (complete) bipartite graphs consisting of two layers. 
One of them is a layer of visible variables, termed visible layer, and the other one is a layer of hidden variables, termed hidden layer.
There are connections between visible variables and hidden variables, and any interlayer connections are not allowed. 

We denote the sets of labels of visible variables and hidden variables by $\Omega$ and $H$, respectively, and we denote   
the state variable of visible variable $i\in \Omega$ by $x_i$ and the state variable of hidden variable $j\in H$ by $h_j$. 
All state variables are binary random variables that take $+1$ or $-1$. 
The joint distribution of RBM is represented by 
\begin{align}
P_{\mrm{RBM}}(\bm{x},\bm{h} \mid \bm{\theta})
&\propto 
\exp\Big(\sum_{i\in \Omega} \alpha_i x_i + \sum_{j\in H}\beta_j h_j\nn
\aleq + \sum_{i\in \Omega}\sum_{ j\in H}w_{i,j}x_ih_j \Big).
\label{eq:RBM}
\end{align}
The parameters $\bm{\alpha} = \{\alpha_i\mid i\in \Omega\}$ and $\bm{\beta} = \{\beta_j\mid j\in H\}$ are biases 
(or sometimes called thresholds) for visible variables and 
hidden variables, respectively, 
and the parameters $\bm{w} =\{w_{i,j} \mid i\in \Omega,\> j\in H\}$ are weights of connections between the visible variables and the hidden variables.
In equation (\ref{eq:RBM}), we denote $\bm{\theta} = \bm{\alpha}\cup \bm{\beta}\cup \bm{w}$ for a short description.

Given an empirical distribution $Q(\bm{x})$ for the visible variables, 
the log-likelihood of RBM in equation (\ref{eq:RBM}) is expressed as 
\begin{align}
\mcal{L}_{\mrm{ML}}(\bm{\theta})=\sum_{\bm{x}}Q(\bm{x}) \ln P_{\mrm{RBM}}(\bm{x} \mid \bm{\theta}),
\label{eq:ML-RBM}
\end{align}
where $P_{\mrm{RBM}}(\bm{x} \mid \bm{\theta})$ is the marginal distribution obtained by
$
P_{\mrm{RBM}}(\bm{x} \mid \bm{\theta})=\sum_{\bm{h}}P_{\mrm{RBM}}(\bm{x},\bm{h} \mid \bm{\theta}).
$
The marginal distribution can be explicitly expressed as
$
P_{\mrm{RBM}}(\bm{x} \mid \bm{\theta})\propto
\exp\big(-E_{\mrm{RBM}}(\bm{x} \mid \bm{\theta}) \big),
$
where
$
E_{\mrm{RBM}}(\bm{x} \mid \bm{\theta})
:=-\sum_{i\in \Omega} \alpha_i x_i 
- \sum_{j \in H} \ln \mcal{C}_j(\bm{x}\mid \bm{\theta})
$
and
$
\mcal{C}_j(\bm{x}\mid \bm{\theta}):=\cosh\big(\beta_j + \sum_{i\in \Omega}w_{i,j}x_i\big)
$.
The CL estimation can be applied to the RBM. 
Indeed, the PL estimation for the RBM was introduced~\cite{RBM_PL2010}.
By applying equation (\ref{eq:CL-k}) to equation (\ref{eq:ML-RBM}), 
we can express the $k$th-order CL for the RBM as 
\begin{align}
&\mcal{L}_{\mcal{F}_k}(\bm{\theta})\nn
&=\Lambda_{\mcal{F}_k}\sum_{c \in \mcal{F}_k}\sum_{i\in c} \alpha_i \ave{x_i}_\mcal{D} 
+ \sum_{j \in H}\ave{\ln \mcal{C}_j(\bm{x}, \bm{\theta})}_{\mcal{D}}\nn
\aleq
-\Lambda_{\mcal{F}_k}\sum_{c \in \mcal{F}_k}\Big<\ln \sum_{\bm{x}_c} \exp\big(-E_{\mrm{RBM}}^{(c)}(\bm{x} \mid \bm{\theta})\big)\Big>_{\mcal{D}},
\label{eq:CLk-RBM}
\end{align}
where the notation $\ave{\cdots}_{\mcal{D}}$ denotes the average over the empirical distribution and
$
E_{\mrm{RBM}}^{(c)}(\bm{x} \mid \bm{\theta}):=-\sum_{i\in c} \alpha_i x_i - \sum_{j \in H} \ln \mcal{C}_j(\bm{x}\mid \bm{\theta})
$.
The gradients, $\Delta_{\theta}^{(k)}:=\partial \mcal{L}_{\mcal{F}_k}(\bm{\theta})/ \partial \theta$, with respect to the parameters 
$\alpha_i$, $\beta_j$ and $w_{i,j}$ are
\begin{align}
\Delta_{\alpha_i}^{(k)}&\propto  \ave{x_i}_\mcal{D} 
-|\mcal{F}_k(i)|^{-1}\sum_{c \in \mcal{F}_k(i)} \ave{x_i}_c,
\label{eq:gradient-a}\\
\Delta_{\beta_j}^{(k)}&\propto \ave{\mcal{T}_j(\bm{x}\mid \bm{\theta})}_{\mcal{D}}
-\Lambda_{\mcal{F}_k}\sum_{c \in \mcal{F}_k}\ave{\mcal{T}_j(\bm{x}\mid \bm{\theta})}_c
\label{eq:gradient-b}
\end{align}
and 
\begin{align}
\Delta_{w_{i,j}}^{(k)}&\propto\ave{x_i \mcal{T}_j(\bm{x}\mid \bm{\theta})}_{\mcal{D}}
-\Lambda_{\mcal{F}_k}\sum_{c \in \mcal{F}_k} \ave{x_i \mcal{T}_j(\bm{x}\mid \bm{\theta})}_c,
\label{eq:gradient-w}
\end{align}
respectively, 
where the notation $\mcal{F}_k(i)$ is the subset of $\mcal{F}_k$ whose all blocks include $i$, 
i.e. $\mcal{F}_k(i):=\{c \mid i \in c \in \mcal{F}_k\}$, the notation $\ave{\cdots}_c$ is defined as
\begin{align*}
\ave{\cdots}_c:=\bigg< \frac{\sum_{\bm{x}_c}(\cdots)\exp\big(-E_{\mrm{RBM}}^{(c)}(\bm{x} \mid \bm{\theta})\big)}
{\sum_{\bm{x}_c} \exp\big(-E_{\mrm{RBM}}^{(c)}(\bm{x} \mid \bm{\theta})\big)} \bigg>_{\mcal{D}},
\end{align*}
for the block $c$, 
and $\mcal{T}_j(\bm{x}\mid \bm{\theta}):=\tanh\big(\beta_j + \sum_{i\in \Omega}w_{i,j}x_i\big)$. 
The computational cost that is required to compute all of them is $O(n^kM|H|)$. 
Note that, when $k=n$, the gradients (\ref{eq:gradient-a})--(\ref{eq:gradient-w}) yield the gradients of the true log-likelihood.

\SubSection{Numerical Experiments}

In this section, we show results of numerical experiments using synthetic data. 
We use an RBM consisting of 5 visible variables and 10 hidden variables as the learning machine, and we generate $M=70$ data from an RBM consisting of 5 visible variables and 17 hidden variables by using the Markov chain Monte Carlo method. 
In the generative RBM, we set $\alpha_i = 0.1$, $\beta_j = - 0.1$ and $w_{i,j} = 0.2$ for all $i$ and $j$. 
We compare the first-, the second- and the third-order CL estimation with the exact ML estimation. 
We maximize the CLs, i.e. $\mcal{L}_{\mcal{F}_1}(\bm{\theta})$, $\mcal{L}_{\mcal{F}_2}(\bm{\theta})$ and $\mcal{L}_{\mcal{F}_3}(\bm{\theta})$, 
and the true log-likelihood, $\mcal{L}_{\mrm{ML}}(\bm{\theta})$, by using the gradient ascent method (GAM) with the update rate of 0.1. 
In the four different GAMs, the same initial values of parameters that are randomly generated are used. 

Figure \ref{fig:CL} shows that the plot of the CLs shown in equation (\ref{eq:CLk-RBM}) and the true log-likelihood shown in (\ref{eq:ML-RBM}) 
against the number of iterations of GAMs with the gradients (\ref{eq:gradient-a})--(\ref{eq:gradient-w}).
\begin{figure}[hbt]
\begin{center}
\includegraphics[height=4.5cm]{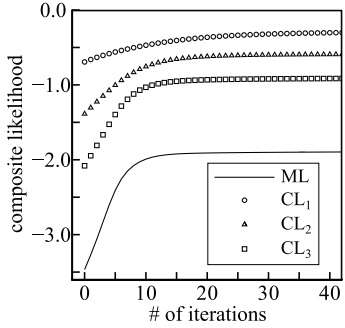} 
\end{center}
\caption{Plot of composite likelihoods against the number of iterations of GAMs. Each point is averaged over 30 trials.}
\label{fig:CL}
\vspace{-5mm}
\end{figure}
In this plot, the ``ML'' is the true log-likelihood obtained by the exact ML estimation 
and ``CL$_k$'' are the $k$th-order CLs, $\mcal{L}_{\mcal{F}_k}(\bm{\theta})$, obtained by the $k$th-order CL estimations. 
One can see that the CL approach the true log-likelihood as $k$ increase.

\begin{figure}[hbt]
\begin{center}
\includegraphics[height=4.5cm]{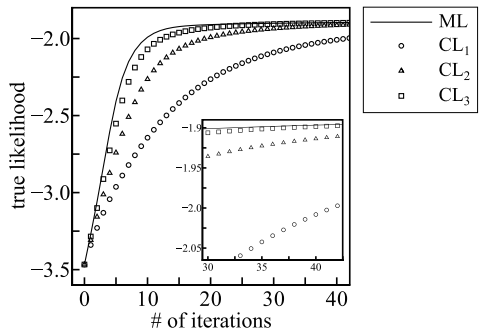} 
\end{center}
\caption{Plot of true log-likelihoods against the number of iterations of GAMs. Each point is averaged over 30 trials.}
\label{fig:ML}
\vspace{-5mm}
\end{figure}
Figure \ref{fig:ML} shows the plot of the true log-likelihoods, $\mcal{L}_{\mrm{ML}}(\bm{\theta})$, against the number of iterations of GAMs with the gradients (\ref{eq:gradient-a})--(\ref{eq:gradient-w}).
In the plot, the ``ML'' is the true log-likelihood with the parameters calculated by the exact ML estimation and 
the ``CL$_k$'' are the true log-likelihoods with the parameters calculated by $k$th-order CL estimations. 
One can see that higher-order CL estimations give better and faster convergence. 

After 50000 iterations, the average values (averaged over 30 trials) of the true log-likelihood obtained by the exact ML estimation, the first-, the second- and the third-order CL estimation
are $-1.741$, $-1.796$, $-1.742$ and $-1.741$, respectively. 
Table \ref{tab:MAD} shows the mean absolute errors (MADs) of the estimations, $\bm{\alpha}$, $\bm{\beta}$ and $\bm{w}$, between obtained by the exact ML estimation and by the $k$th-order CL estimation 
after 50000 iterations. Each MAD is averaged over 30 trials.
\begin{table}[htbp]
\vspace{-5mm}
\begin{center}
\caption{MADs of estimations between obtained by exact ML estimation and 
by $k$th-order CL estimation after 50000 iterations.}\label{tab:MAD}
\vspace{2mm}
\begin{tabular}{|c||c|c|c|} \hline
      &$\bm{\alpha}$ &$\bm{\beta}$ &$\bm{w}$ \\ \hline\hline
$k=1$ &0.377    &0.431    &0.360 \\ \hline
$k=2$ &0.223    &0.223    &0.192 \\ \hline
$k=3$ &0.128    &0.114    &0.103 \\ \hline
\end{tabular}
\end{center}
\vspace{-5mm}
\end{table}
One can see that higher-order CL estimations give quantitatively better estimations.

\Section{Conclusion}

In this paper, we introduced the systematic choice of blocks for the maximum CL estimation 
which guarantees that the $k$th-order CL monotonically approaches the true log-likelihood with the increase of $k$. 
This property does not depend on details of graphical models. 

Furthermore, we applied our CL method to learning of RBMs and formulate learning algorithm explicitly. 
In our numerical experiments for synthetic data, we made sure that the higher-order CLs have better performances. 
As we have seen in section \ref{sec:RBM}, the computational cost increases when higher-order CLs are employed. 
Nonetheless, it is possible to trade off computation time for increased accuracy by switching to higher-order CLs.

\subsubsection*{acknowledgment}
This work was partly supported by Grants-In-Aid (No. 21700247, No. 22300078 and No. 23500075) 
for Scientific Research from the Ministry of Education, Culture, Sports, Science and Technology, Japan.

\balance
\bibliographystyle{latex12}
\bibliography{latex12}

\end{document}